\documentclass[runningheads]{llncs}
\usepackage[T1]{fontenc}
\usepackage[utf8]{inputenc} 
\usepackage{hyperref}       
\usepackage{url}            
\usepackage{booktabs}       
\usepackage{amsfonts}       
\usepackage{nicefrac}       
\usepackage{microtype}      
\usepackage{xcolor}         
\usepackage{amsmath,amssymb}
\usepackage{mathtools,cases}
\usepackage{comment}
\usepackage{subcaption}
\usepackage{wrapfig}

\DeclareMathOperator*{\argmin}{arg\,min}

\newcommand{\X}{\boldsymbol{X}}
\newcommand{\V}{\boldsymbol{V}}
\newcommand{\W}{\boldsymbol{W}}
\newcommand{\Y}{\boldsymbol{Y}}
\newcommand{\Q}{\boldsymbol{Q}}
\newcommand{\Z}{\boldsymbol{Z}}
\newcommand{\M}{\boldsymbol{M}}
\newcommand{\A}{\boldsymbol{A}}
\renewcommand{\H}{\boldsymbol{H}}
\newcommand{\B}{\boldsymbol{B}}
\newcommand{\D}{\boldsymbol{D}}
\newcommand{\R}{\boldsymbol{R}}
\renewcommand{\P}{\boldsymbol{P}}
\newcommand{\I}{\boldsymbol{I}}
\renewcommand{\L}{\boldsymbol{\Lambda}}
\renewcommand{\S}{\boldsymbol{\Sigma}}

\usepackage{graphicx}

\begin{document}
\title{Can multimodal representation learning by alignment preserve modality-specific information?}
\titlerunning{Can alignment preserve modality-specific information?}
%
\author{Romain Thoreau$^{1}$ \and Jessie Levillain$^{1,2}$ \and
Dawa Derksen$^{1}$
}

\authorrunning{R. Thoreau et al.}

\institute{$^1$CNES, Toulouse, France, \email{\{name.surname\}@cnes.fr}\\ $^2$INSA-IMT, Toulouse, France}
\maketitle              

\begin{abstract}
Combining multimodal data is a key issue in a wide range of machine learning tasks, including many remote sensing problems. 
In Earth observation, early multimodal data fusion methods were based on specific neural network architectures and supervised learning. 
Ever since, the scarcity of labeled data has motivated self-supervised learning techniques. State-of-the-art multimodal representation learning techniques leverage the spatial alignment between satellite data from different modalities acquired over the same geographic area in order to foster a semantic alignment in the latent space. 
In this paper, we investigate how this methods can preserve task-relevant information that is not shared across modalities. 
First, we show, under simplifying assumptions, when alignment strategies fundamentally lead to an information loss.
Then, we support our theoretical insight through numerical experiments in more realistic settings.
With those theoretical and empirical evidences, we hope to support new developments in contrastive learning for the combination of multimodal satellite data.
Our code and data is publicly available at \url{https://github.com/Romain3Ch216/alg_maclean_25}.

\end{abstract}

\section{Introduction}

\thispagestyle{firstpage}

In a wide range of machine learning tasks, combining multiple data sources (\textit{e.g.} audio and images, optical and radar images in remote sensing) has great potential \cite{baltruvsaitis2018multimodal}. 
For instance, combining multi-source satellite data has proved valuable for the detection of pluvial flood-induced damages \cite{cerbelaud2023mapping}, for the quantification of toxic plumes released by industries and power plants \cite{calassou2024quantifying}, or for tree species identification \cite{astruc2024omnisat}.
To this end, self-supervised contrastive learning (CL) algorithms have recently gained considerable traction in order to learn task-agnostic multimodal representations \cite{liang2022foundations}.
Standard CL algorithms are based on the \textit{multi-view redundancy} assumption, stating that task-relevant information is shared across modalities \cite{sridharan2008information,tsaiself}.
In Earth observation, this assumption has motivated methods that leverage the spatial alignment of data acquired by different sensors over the same geographic area in order to align multimodal representations in the latent space (\textit{e.g.} \cite{marsocci2024cross} and \cite{astruc2024omnisat}). 
Obviously, the \textit{multi-view redundancy} assumption does not hold in many real-world multimodal problems, including classification or regression tasks in remote sensing.
For example, altimetry and optical satellite data provide complementary and independent information on forests properties, such as canopy height and leaf traits, respectively.
Therefore, a fundamental research question is \textit{whether contrastive learning algorithms can preserve modality-specific information, beyond capturing modality-generic information through alignment?}
Recent works have taken an information-theoretic point of view in order to challenge and improve the applicability of CL algorithms to multimodal problems (\textit{e.g.}  \cite{liang2023factorized}, \cite{dufumier2025align}).
In this paper, we take another perspective on multimodal CL and introduce a potential framework in order to investigate success and limitations in multimodal representation learning by alignment. Specifically, our contributions are summarized below:
\begin{itemize}
    \item We introduce the notion of $\sigma$-informativeness, close to the idea of information formalized in \cite{zhang2023learning}, in order to identify different regimes in multimodal representation learning,
    \item We demonstrate that, in a linear setting, standard CL can lead to a fundamental loss of modality-specific information,
    \item Through numerical experiments on simulated and real remote sensing datasets, we show that our theoretical insights are also informative in non-linear settings.
\end{itemize}


\noindent
The paper is organized as follows. In section \ref{sec: setting}, we introduce the necessary mathematical framework for investigating multimodal CL in a linear regime. In section \ref{sec:loss}, we state and prove when representation learning by alignment leads to an information loss. In section \ref{sec:exps}, we extend our theoretical study to a non-linear regime through numerical experiments. In section \ref{sec:related}, we make connections to related work based on information theory. Finally, we discuss about limitations and perspectives in section \ref{sec:conclusion}.  


\section{Multimodal Representation Learning by Alignment}\label{sec: setting}

We consider two data modalities (\textit{e.g.} 1) radar and 2) optical data over a forest area) from which we aim to predict two different targets (\textit{e.g.} 1) canopy height \& above-ground biomass and 2) chlorophyll \& nitrogen concentrations). 
Let us denote $\X_1 \in \mathbb{R}^{N \times D_1}$ the input data from modality 1 and $\X_2 \in \mathbb{R}^{N \times D_2}$ the data from modality 2, where $N$ is the number of training samples, and $D_1, D_2$ are the dimensions of the data. 
Besides, let us denote $\Y_i \in \mathbb{R}^{N \times C_i}$ the targets related to the modality $i$. 
We consider a setting that displays \textit{multi-view non-redundancy}: we assume that modality $\X_i$ is more informative of targets $\Y_i$ than modality $\X_j$, and that it contains little information about targets $\Y_j$.
To be precise, we introduce the definition of $\sigma$-informativeness and reformulate the \textit{multi-view non-redundancy} assumption.

\begin{definition}[$\sigma$-informativeness] Data $\X \in \mathbb{R}^{N \times D}$ is $\sigma$-informative ($ 0 \leq \sigma \leq 1$) with respect to the targets $\Y \in \mathbb{R}^{N \times C}$ if $\|\hat{\Y} - \Y \|_F^2 = NC(1 - \sigma)$, where $N$ is the number of samples, $C$ is the dimension of the targets, and $\hat{\Y}$ is the prediction of the ordinary least squares (OLS) estimator: $\hat{\Y} = \X(\X^T\X)^{-1}\X^T\Y$.
\end{definition}

\noindent
Under the assumption that the targets have a unit variance over every dimensions, $\sigma$ is the coefficient of determination averaged over the $C$ targets. 
The intuition behind $\sigma$-informative data is close to the idea of information formalized in \cite{zhang2023learning}: $\sigma$ measures the amount of information contained in the data $\X$ that can be linearly exploited in order to predict the targets $\Y$.

\begin{definition}[Multi-view non-redundancy] Let the data modality $\X_i$ be $\sigma_{ij}$-informative of targets $\Y_j$. For $i \neq j$, $\sigma_{ii} > \sigma_{ji}$.
\end{definition}

\noindent
Under the \textit{multi-view non-redundancy} assumption, both modalities are useful in order to predict the targets. Note that this definition of \textit{multi-view non-redundancy} is analogue, but not equivalent, to the definition used in related work (\textit{e.g.} \cite{liang2023factorized}). \\

\noindent
\textbf{Framework} In this work, we aim to gain theoretical and empirical insights about the mechanisms of representation learning by alignment. 
As a starting point, we will build intuition in a linear regime where we consider alignment as a regression task, consisting in predicting the representation of modality 2 given the representation of modality 1. 
This modeling choice will be further discussed in section \ref{sec:conclusion}.
In our framework, linear encoders, denoted as $\V_i \in \mathbb{R}^{D_i \times K}$, compute the latent representations $\Z_i = \X_i \V_i$, where $K$ is the dimension of the latent space. 
Predictors, denoted as $\W_i \in \mathbb{R}^{K \times C_i}$,  map the representations to the targets $\Y_i$.
Finally, a linear head $\Q_1 \in \mathbb{R}^{K \times K}$ maps the representations of modality 1 to the representations of modality 2.
We formalize the learning problem as the following optimization problem:
\begin{align}
	\min_{\V_1, \V_2, \W_1, \W_2, \Q_1} & \overbrace{\|\X_1\V_1\W_1 - \Y_1\|_F^2}^{\scriptsize{\mbox{prediction loss 1}}} + \overbrace{\|\X_2\V_2\W_2 - \Y_2\|_F^2}^{\scriptsize{\mbox{prediction loss 2}}} \nonumber
	 \\
	 & + \lambda \underbrace{\|\X_1 \V_1\Q_1 - \X_2\V_2 \|_F^2}_{\scriptsize{\mbox{alignment loss}}}  \label{eq:optim_pb} \\
	\mbox{subject to } & \V_2, \W_2 = \argmin_{\V, \W} \|\X_2 \V \W - \Y_2 \|_F^2 \nonumber
\end{align}
where $\lambda > 0$ controls the trade-off between the alignment and the prediction losses. 
The key question is \textit{whether the prediction and the alignment losses can be jointly minimized}?
In other words, \textit{are optimal model parameters independent of $\lambda$?} \vspace{0.1cm}

\noindent
In the following, we assume that $\X_1$ and $\X_2$ have full rank, that the top-$K$ eigenvalues of $\X_i\X_i^T$ have a multiplicity of one.
We also assume that the eigenvalues of $\Y_i\Y_i^T$, denoted as $\lambda_k$, are such that $\lambda_k > \sum_{i=k+1} \lambda_i$ for every $k$.
Besides, we use the notation $\M = \P_{\M} \D_{\M} \P_{\M}^T$ for the eigendecomposition of a symmetric semi-positive matrix $\M$. The columns of $\P_{\M}$ are the eigenvectors of $\M$ and the diagonal elements of $\D_{\M}$ are its eigenvalues. Besides, $(\M)_{\cdot, k}$ denotes the $k^{\mbox{th}}$ column of $\M$.

\section{How Alignement Leads to an Information Loss} \label{sec:loss}




In this section, we show that the minimization of the alignment loss implies a trade-off between the prediction loss 1 and the prediction loss 2. 
We start by deriving analytical solutions of the learning problem \eqref{eq:optim_pb} in Theorem \ref{th:1}. Then, we formalize in Theorem \ref{th:2} how the encoder $\V_1$ cannot minimize the alignment loss and the prediction loss 1 at the same time, resulting in representations $\Z_1$ that are less informative of the targets $\Y_1$ than the data $\X_1$.

\begin{theorem}[Solutions to the learning problem \eqref{eq:optim_pb}] \label{th:1} The objective function of the optimization problem \eqref{eq:optim_pb} is minimized for $\V_1^\ast, \V_2^\ast, \W_1^\ast, \W_2^\ast, \Q_1^\ast$ such that:
\begin{align*}
	\mbox{\normalfont span}(\V_1^\ast) & = \mbox{\normalfont span}\big(\P_{\X_1^T\X_1} \D_{\X_1^T\X_1}^{-\frac{1}{2}} (\P_{\H_1})_{\cdot, 1:K}\big) \label{eq:spans}\\
	\mbox{\normalfont span}(\V_2^\ast) & = \mbox{\normalfont span}\big(\P_{\X_2^T\X_2} \D_{\X_2^T\X_2}^{-\frac{1}{2}} (\P_{\H_2})_{\cdot, 1:K}\big) \\
	\W_1^\ast & = (\V_1^{\ast T}\X_1^T\X_1\V_1^\ast)^{-1}\V_1^{\ast T}\X_1^T\Y_1 \\
	\W_2^\ast & = (\V_2^{\ast T}\X_2^T\X_2\V_2^\ast)^{-1}\V_2^{\ast T}\X_2^T\Y_2 \\
	\Q_1^\ast & = (\V_1^{\ast T}\X_1^T\X_1\V_1^\ast)^{-1}\V_1^{\ast T}\X_1^T\X_2\V_2^\ast
\end{align*}
where $\mbox{\normalfont span}(\M)$ stands for the linear subspace generated by the columns of $\M$, and:
\begin{align*}
	\begin{cases}
	\H_1 := \D_{\X_1^T\X_1}^{-\frac{1}{2}} \P_{\X_1^T\X_1}^T \A \P_{\X_1^T\X_1} \D_{\X_1^T\X_1}^{-\frac{1}{2}} \\
	\A := \X_1^T(\Y_1\Y_1^T + \lambda\Z_2 \Z_2^T)\X_1 \\
	\Z_2 := \X_2\V_2^\ast \\
	\P_{\H_2} = (\P_{\X_2\X_2^T})_{\cdot, 1:D}^T (\P_{\Y_2\Y_2^T})_{\cdot, 1:D}
	\end{cases}
\end{align*}

\end{theorem}

\begin{proof}
The proof is given in Appendix \ref{app:th1}.
\end{proof}

\begin{theorem}[Information loss by alignment]
Let us consider that data $\X_i$ is $\sigma_{ij}$-informative of targets $\Y_j$, and that $\tilde{Y}_{iK}$ is $\sigma_{iK}$-informative of $\Y_i$, where $\tilde{Y}_{iK}$ is the projection of $\Y_i$ on the top-K eigenspaces of $\Y_i^T\Y_i$. \vspace{0.2cm}

\noindent
Under the assumptions that i) $\sigma_{22} \geq \sigma_{2K}$, and ii) $\sigma_{21} < \sigma_{1K}$, then the solution $\V_1^\ast$ of problem \eqref{eq:optim_pb} does not minimize the prediction loss 1 neither the alignment loss. Furthermore, representations $\Z_1$ are $\sigma_{11}^z$-informative of $\Y_1$, with $\sigma_{11}^z < \sigma_{11}$. \label{th:2}
\end{theorem}

\noindent
Basically, assumption i) characterizes how much $\X_2$ is informative of $\Y_2$.
For a usual value of $K$, $\sigma_{2K}$ will be high, and $\X_2$ highly informative of $\Y_2$.
On the contrary, assumption ii) states that $\X_2$ is barely informative of $\Y_1$.
Theorem \ref{th:2} implies that if the prediction loss 2 is minimized, then the prediction task 1 cannot be minimized.
In other words, multimodal representation learning by alignment leads to a fundamental trade-off between the two downstream tasks.
This trade-off translates into a fundamental loss of modality-specific information (in this setting, modality 1).

\begin{proof}
The proof is made by contradiction. 
The idea is to show that if $\V_1^\star$ is the optimal solution of the prediction task 1 and the alignment task, then $\sigma_{21} \geq \sigma_{1K}$, which contradicts assumption ii).
The proof is divided into five steps:
\begin{enumerate}
    \item We assume that $\V_1^\star$ is the optimal solution of the prediction task 1 and the alignment task, and show that, as a result, the intersection of the top-K eigenspaces of $\Z_2\Z_2^T$ and $\Y_1\Y_1^T$ is of dimension $K$,
    \item We show that $\sigma_{22} \geq \sigma_{2K}$ implies that the intersection of the top-K eigenspaces of $\X_2\X_2^T$ and $\Y_2\Y_2^T$ is of dimension $K$,
    \item From step 2, we show that the intersection of the top-K eigenspaces of $\Z_2\Z_2^T$ and $\X_2\X_2^T$ is of dimension $K$,
    \item From step 1 and step 3, we deduce that the intersection of the top-K eigenspaces of $\Y_1\Y_1^T$ and  $\X_2\X_2^T$ is of dimension $K$.
    \item Finally, we show that this implies that $\sigma_{21} \geq \sigma_{1K}$, which contradicts assumption ii).
\end{enumerate}
We include a detailed proof in Appendix \ref{app:th2}.
\end{proof}

\noindent

\section{Preliminary numerical experiments} \label{sec:exps}

In order to extend our theoretical insights in more realistic settings with non-linear encoders, we made preliminary numerical experiments on a controlled synthetic dataset and on a remote sensing dataset.

\subsection{Controlled experiments on a synthetic dataset}

\begin{figure}[h]
    \centering
    \begin{subfigure}{0.45\textwidth}
        \center
        \includegraphics[width=0.8\textwidth]{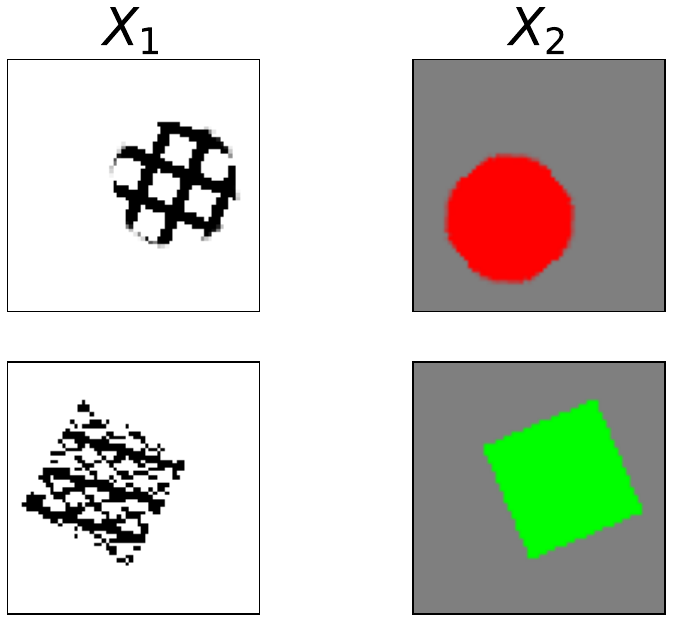}
        \caption{Data samples from the synthetic dataset. Three generative factors -- shape, texture, and color -- control the data generation across two modalities.} \label{fig:data_samples}
    \end{subfigure}
    \hfill
    \begin{subfigure}{0.45\textwidth}
        \center
        \includegraphics[width=\textwidth]{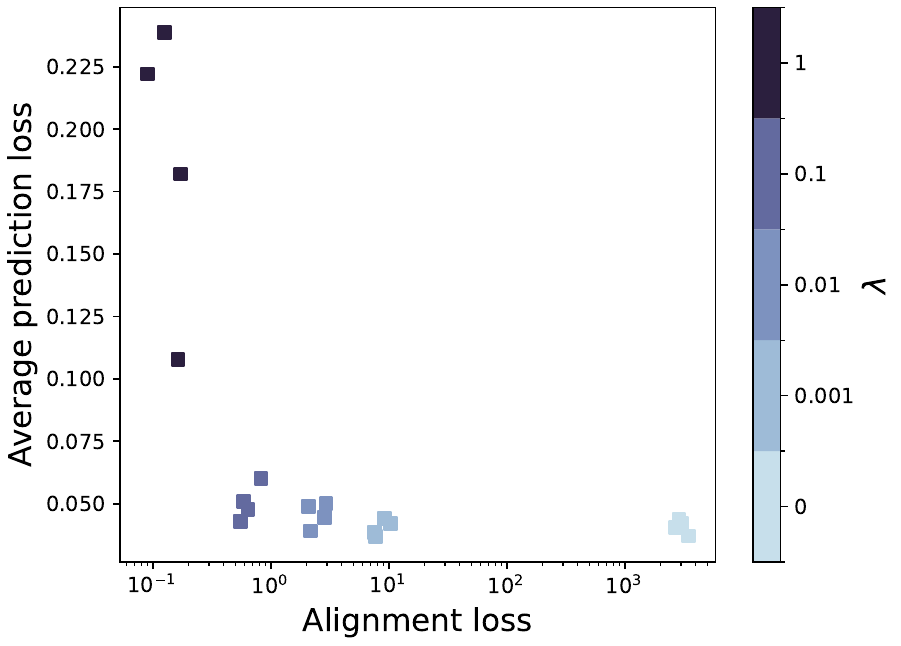}
        \caption{Prediction loss VS alignment loss for several values of the coefficient $\lambda$ and random network initializations.} \label{fig:alg_vs_pred_exp1}
    \end{subfigure}
    \label{fig:trifeatures_exp}
    \caption{Controlled experiments on a synthetic dataset}
\end{figure}

\noindent
In order to precisely control how informative the modalities are with respect to the targets, we designed a multimodal dataset inspired by \cite{dufumier2025align}. The dataset is a variation of the Trifeature dataset \cite{NEURIPS2020_71e9c662}: it comprises two data modalities that share modality-generic features (shape) and carry modality-specific features (texture and color). Fig. \ref{fig:data_samples} shows random data samples. Targets $\Y_1$ (resp. $\Y_2$) comprise ground truth information about the shape and the texture (resp. about the shape and the color). Our experiment consists in training non-linear encoders $f_\theta^1$ and $f_\theta^2$ in order to minimize the weighted combination of the prediction losses and of the alignment loss, for several values of the coefficient $\lambda \in [0, 1]$ that controls the trade-off. Although the encoders are non-linear, it is common to keep linear classification heads \cite{caron2021emerging}, leading to the following loss function:
\begin{align}
    \mathcal{L} = \|f_{\theta}(\X_1)\W_1 - \Y_1\|_F^2 + \|f_{\theta}(\X_2)\W_2 - \Y_2\|_F^2
	 + \lambda \|f_{\theta}(\X_1)\Q_1 - f_{\theta}(\X_2)\|_F^2 
\end{align}
For each value of $\lambda$, we made several random parameter initializations. The best training prediction losses are plotted against the alignment loss in Fig. \ref{fig:alg_vs_pred_exp1}. It shows that the more aligned the data, the worse the prediction training performances.
Although, for small values of $\lambda$, the difference of prediction losses is not significant.
Thus, this controlled experiment seems to support that information loss by alignment can also
happen with non-linear encoders.

\subsection{Experiments on a remote sensing dataset}

\begin{figure}
    \center
    \includegraphics[width=0.7\textwidth]{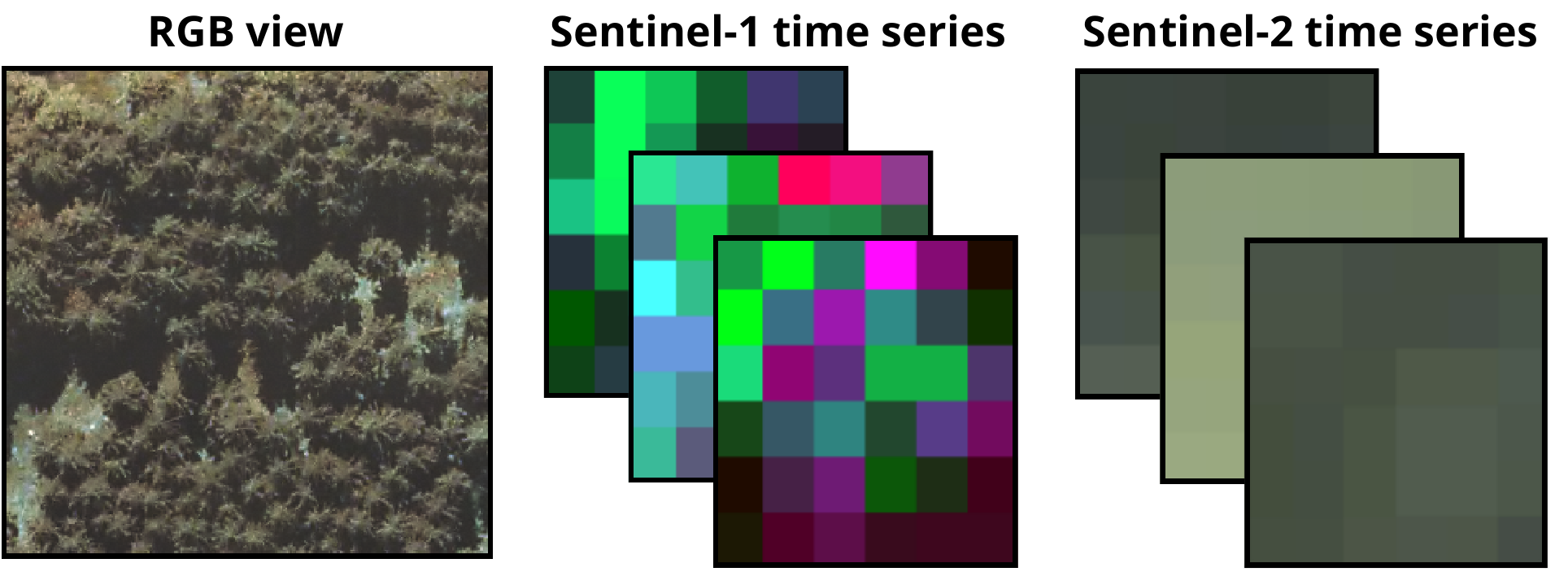}
    \caption{Illustration of a data sample from the TreeSatAI-Time-Series dataset \cite{astruc2024omnisat}.}
    \label{fig:ts}
\end{figure}

\noindent
We selected the TreeSatAI-Time-Series dataset \cite{astruc2024omnisat} as a real-world multimodal dataset.
It comprises very high resolution RGB-NIR images, Sentinel-1 (radar) image time series, and Sentinel-2 (optical) image time series acquired over forest areas. 
The ground truth consists in multi-label annotations of tree species at the image level.
As a multimodal model, we used the OmniSat architecture \cite{astruc2024omnisat}, though with much less parameters than the original model, as it better suited our computational resources.
Also, we selected a subset of the data comprising the most frequent classes, and only kept the Sentinel-1 radar image time series and the Sentinel-2 multispectral image time series. \\

\begin{wrapfigure}{r}{0.5\textwidth}
    \includegraphics[width=0.5\textwidth]{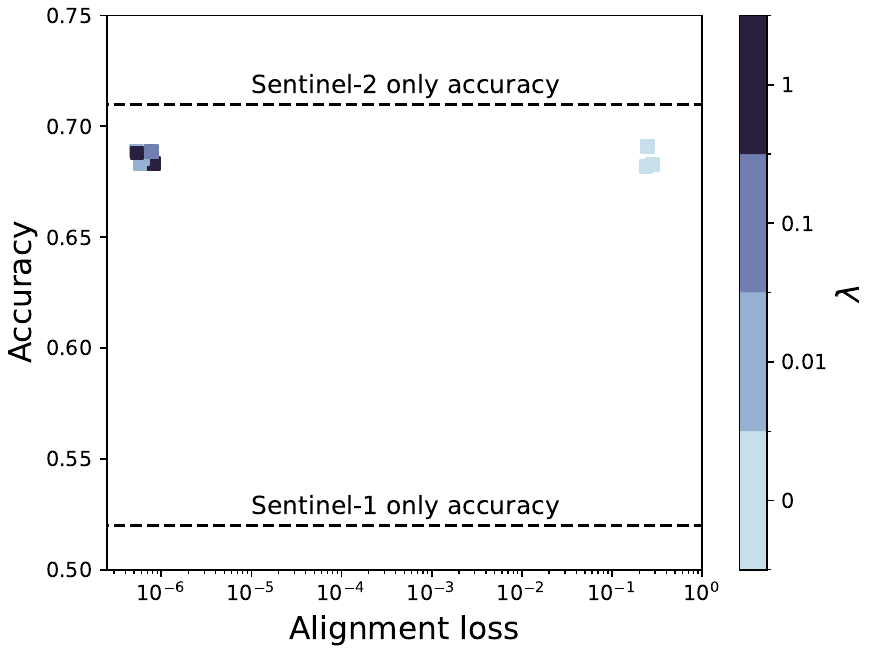}
    \captionof{figure}{Train accuracy VS alignment loss for several values of the coefficient $\lambda$ and random network initializations. Horizontal lines denote the accuracy reached by models trained on a single modality.} \label{fig:alg_vs_acc_exp2}
\end{wrapfigure}

\noindent
In this real-world setting, the labeled data consists in one target only (the multi-label classes), and we have no prior knowledge about the informativeness of each modality with respect to the target.
In other words, assessing whether the \textit{multi-view non-redundancy} assumption holds is not possible.
However, a supervised training with a single modality shows that the multispectral image time series likely provide much more information about the tree species than the radar image time series, with +20\% accuracy on the training set.
That being said, the radar data still could provide useful information on some samples. \\

\noindent
Fig. \ref{fig:alg_vs_acc_exp2} shows the train accuracy against the alignment loss for several values of $\lambda$ and random network initializations.
It shows that alignment had no influence on the prediction accuracy: it did not decrease the model performance.
We assume that radar data did not provide complementary information with respect to optical data.
As a result, the encoder dedicated to optical data likely learned useful information for the classification task, while the encoder dedicated to radar data likely fitted the multispectral latent representations.
Nevertheless, we would like to stress that these experimental results do not call into question the value of radar images for learning generic latent representations through pre-training.

\section{Related work} \label{sec:related}

Several works have investigated multimodal CL throuhg the lens of information theory. In particular, the theory of Partial Information Decomposition (PID) \cite{williams2010nonnegative} allows to model the information that two variables $X_1, X_2$ provide about a target variable $Y$. 
PID decomposes information into three forms of interaction: redundant information in $X_1$ and $X_2$, unique information in $X_1$ / $X_2$, or synergistic information that only emerges when both $X_1$ and $X_2$ are observed.
\cite{liang2023factorized} use the PID formalism in order to demonstrate that standard CL methods are suboptimal under the \textit{multi-view non-redundancy assumption}\footnote{Formalized as follows: $\exists \epsilon > 0$ such that $I(X_1;Y \vert X_2) > \epsilon$ or $I(X_2;Y \vert X_1) > \epsilon$.} and under the assumption that the learned representations $Z_i$ carry redundant information about the targets with data $X_j$\footnote{$I(Z_1;Y \vert X_2) = I(Z_2;Y \vert X_1) = 0$.}. 
In the same vein, \cite{dufumier2025align} show that standard CL methods are limited to only learn the redundant information under the mutli-view redundancy assumption.
In contrast, in this work, we assume multi-view non-redundancy and rely on weaker assumptions in order to demonstrate, in a simplified setting, the suboptimality of standard CL.
In order to go beyond learning redundant features across modalities, \cite{dufumier2025align} introduced an approach, named CoMM, aiming to maximize mutual information between representations of augmented multimodal data. Using the PID formalism, they show that CoMM can extract unique and synergistic information from the different modalities.
Although CoMM was not experimented on remote sensing datasets, it seems like a promising CL algorithm to improve the combination of multimodal satellite data.
\vspace{-0.3cm}

\section{Limitations and perspectives} \label{sec:conclusion}

In this paper, we presented a potential mathematical framework to study multimodal contrastive learning methods.
In a linear regime, we demonstrated that self-supervised learning by alignment leads to a loss of modality-specific information when modalities provide different kind of information about the targets.
Furthermore, preliminary numerical experiments supported our theoretical insight in a non-linear regime. 
However, our theoretical and empirical analyses suffer from multiple limitations that motivate future work and additional numerical experiments.
First, we have modeled the alignment task a regression task, while standard CL methods maximize the mutual information between modalities (in practice, a lower bound using the NCE loss \cite{oord2018representation}).
Our modeling choice may be closer to Joint-Embedding Predictive Architectures \cite{assran2023self} that, given the representation of a context view (\textit{i.e.} one modality) and metadata, aim to predict the representation of a target view (\textit{i.e} another modality).
Second, the assumptions under which Th. \ref{th:2} holds are difficult to verify in practice. In realistic scenarios, the notion of $\sigma$-informativeness is limited to characterize the representations of the data, computed by non-linear mappings, and not the data itself. 
Third, we have demonstrated that CL can lead to a loss of information, but we have not quantified nor qualified this loss. In particular, an open research question is whether alignment focuses on \textit{simple} features at the expense of \textit{complex} features.
Finally, our numerical experiments do not account for realistic scenarios where models are pre-trained through a CL algorithm, and then fine-tuned for a specific downstream task. We have neither studied the generalization performances of the models.
Therefore, we would like to address those limitations in future work.
Besides, we would like to study how the disentanglement of modality-generic and modality-specific features, if possible, could improve CL methods.



%
\bibliographystyle{apalike}
\bibliography{mybibliography}

\begin{thebibliography}{}

\bibitem[Assran et~al., 2023]{assran2023self}
Assran, M., Duval, Q., Misra, I., Bojanowski, P., Vincent, P., Rabbat, M.,
  LeCun, Y., and Ballas, N. (2023).
\newblock Self-supervised learning from images with a joint-embedding
  predictive architecture.
\newblock In {\em Proceedings of the IEEE/CVF Conference on Computer Vision and
  Pattern Recognition}, pages 15619--15629.

\bibitem[Astruc et~al., 2024]{astruc2024omnisat}
Astruc, G., Gonthier, N., Mallet, C., and Landrieu, L. (2024).
\newblock Omnisat: Self-supervised modality fusion for earth observation.
\newblock In {\em European Conference on Computer Vision}, pages 409--427.
  Springer.

\bibitem[Balestriero and LeCun, 2024]{balestrierolearning}
Balestriero, R. and LeCun, Y. (2024).
\newblock How learning by reconstruction produces uninformative features for
  perception.
\newblock In {\em Forty-first International Conference on Machine Learning}.

\bibitem[Baltru{\v{s}}aitis et~al., 2018]{baltruvsaitis2018multimodal}
Baltru{\v{s}}aitis, T., Ahuja, C., and Morency, L.-P. (2018).
\newblock Multimodal machine learning: A survey and taxonomy.
\newblock {\em IEEE transactions on pattern analysis and machine intelligence},
  41(2):423--443.

\bibitem[Calassou et~al., 2024]{calassou2024quantifying}
Calassou, G., Foucher, P.-Y., and L{\'e}on, J.-F. (2024).
\newblock Quantifying particulate matter optical properties and flow rate in
  industrial stack plumes from the prisma hyperspectral imager.
\newblock {\em Atmospheric Measurement Techniques}, 17(1):57--71.

\bibitem[Caron et~al., 2021]{caron2021emerging}
Caron, M., Touvron, H., Misra, I., J{\'e}gou, H., Mairal, J., Bojanowski, P.,
  and Joulin, A. (2021).
\newblock Emerging properties in self-supervised vision transformers.
\newblock In {\em Proceedings of the IEEE/CVF international conference on
  computer vision}, pages 9650--9660.

\bibitem[Cerbelaud et~al., 2023]{cerbelaud2023mapping}
Cerbelaud, A., Blanchet, G., Roupioz, L., Breil, P., and Briottet, X. (2023).
\newblock Mapping pluvial flood-induced damages with multi-sensor optical
  remote sensing: a transferable approach.
\newblock {\em Remote Sensing}, 15(9):2361.

\bibitem[Chen et~al., 2020]{chen2020simple}
Chen, T., Kornblith, S., Norouzi, M., and Hinton, G. (2020).
\newblock A simple framework for contrastive learning of visual
  representations.
\newblock In {\em International conference on machine learning}, pages
  1597--1607. PmLR.

\bibitem[Dufumier et~al., 2025]{dufumier2025align}
Dufumier, B., Castillo-Navarro, J., Tuia, D., and Thiran, J.-P. (2025).
\newblock What to align in multimodal contrastive learning?
\newblock In {\em International Conference on Learning Representations}. ICLR.

\bibitem[Ghojogh et~al., 2019]{ghojogh2019eigenvalue}
Ghojogh, B., Karray, F., and Crowley, M. (2019).
\newblock Eigenvalue and generalized eigenvalue problems: Tutorial.
\newblock {\em arXiv preprint arXiv:1903.11240}.

\bibitem[Hermann and Lampinen, 2020]{NEURIPS2020_71e9c662}
Hermann, K. and Lampinen, A. (2020).
\newblock What shapes feature representations? exploring datasets,
  architectures, and training.
\newblock In Larochelle, H., Ranzato, M., Hadsell, R., Balcan, M., and Lin, H.,
  editors, {\em Advances in Neural Information Processing Systems}, volume~33,
  pages 9995--10006. Curran Associates, Inc.

\bibitem[Ienco and Dantas, 2024]{ienco2024discom}
Ienco, D. and Dantas, C.~F. (2024).
\newblock Discom-kd: Cross-modal knowledge distillation via disentanglement
  representation and adversarial learning.
\newblock In {\em BMVC 2024-35th British Machine Vision Conference}.

\bibitem[Jain et~al., 2022]{jain2022multimodal}
Jain, U., Wilson, A., and Gulshan, V. (2022).
\newblock Multimodal contrastive learning for remote sensing tasks.
\newblock {\em arXiv preprint arXiv:2209.02329}.

\bibitem[Liang et~al., 2023]{liang2023factorized}
Liang, P.~P., Deng, Z., Ma, M.~Q., Zou, J.~Y., Morency, L.-P., and
  Salakhutdinov, R. (2023).
\newblock Factorized contrastive learning: Going beyond multi-view redundancy.
\newblock {\em Advances in Neural Information Processing Systems},
  36:32971--32998.

\bibitem[Liang et~al., 2022]{liang2022foundations}
Liang, P.~P., Zadeh, A., and Morency, L.-P. (2022).
\newblock Foundations and trends in multimodal machine learning: Principles,
  challenges, and open questions.
\newblock {\em arXiv preprint arXiv:2209.03430}.

\bibitem[Marsocci and Audebert, 2024]{marsocci2024cross}
Marsocci, V. and Audebert, N. (2024).
\newblock Cross-sensor self-supervised training and alignment for remote
  sensing.
\newblock {\em arXiv preprint arXiv:2405.09922}.

\bibitem[Oord et~al., 2018]{oord2018representation}
Oord, A. v.~d., Li, Y., and Vinyals, O. (2018).
\newblock Representation learning with contrastive predictive coding.
\newblock {\em arXiv preprint arXiv:1807.03748}.

\bibitem[Sohn et~al., 2014]{NIPS2014_3b8203bb}
Sohn, K., Shang, W., and Lee, H. (2014).
\newblock Improved multimodal deep learning with variation of information.
\newblock In Ghahramani, Z., Welling, M., Cortes, C., Lawrence, N., and
  Weinberger, K., editors, {\em Advances in Neural Information Processing
  Systems}, volume~27. Curran Associates, Inc.

\bibitem[Sridharan and Kakade, 2008]{sridharan2008information}
Sridharan, K. and Kakade, S.~M. (2008).
\newblock An information theoretic framework for multi-view learning.
\newblock In {\em COLT}, number 114, pages 403--414.

\bibitem[Tian et~al., 2020]{tian2020makes}
Tian, Y., Sun, C., Poole, B., Krishnan, D., Schmid, C., and Isola, P. (2020).
\newblock What makes for good views for contrastive learning?
\newblock {\em Advances in neural information processing systems},
  33:6827--6839.

\bibitem[Tsai et~al., 2020]{tsaiself}
Tsai, Y.-H.~H., Wu, Y., Salakhutdinov, R., and Morency, L.-P. (2020).
\newblock Self-supervised learning from a multi-view perspective.
\newblock In {\em International Conference on Learning Representations}.

\bibitem[Williams and Beer, 2010]{williams2010nonnegative}
Williams, P.~L. and Beer, R.~D. (2010).
\newblock Nonnegative decomposition of multivariate information.
\newblock {\em arXiv preprint arXiv:1004.2515}.

\bibitem[Zhang and Bottou, 2023]{zhang2023learning}
Zhang, J. and Bottou, L. (2023).
\newblock Learning useful representations for shifting tasks and distributions.
\newblock In {\em International Conference on Machine Learning}, pages
  40830--40850. PMLR.

\end{thebibliography}
\nocite{*}

\newpage

\appendix

\section{Appendix}

\subsection{Proof of Theorem \ref{th:1} \label{app:th1}}

The proof of Theorem \ref{th:1} is, in part, analogue to the proof of Theorem 1 in \cite{balestrierolearning}, where the objective function balances between one prediction term and a reconstruction term, rather than between two prediction terms and an alignment term, as it is in our case. \\

\noindent
First, we express optimal parameters $\W_1^\ast$ and $\Q_1^\ast$ as functions of $\V_1$. Second, we show that a closed-form expression of the optimal parameter $\V_1^\ast$ can be obtained by solving a generalized eigenvalue problem. \\

\noindent
Let us denote the objective function of problem \eqref{eq:optim_pb} as $\mathcal{L}_1$. Recalling that $\|\M\|_F^2 = Tr(\M^T\M)$ and that the trace is invariant to cyclic permutations, we can expand $\mathcal{L}_1$ as follows:
\begin{align*}
	\mathcal{L}_1 & := Tr(\W_1^T\V_1^T\X_1^T\X_1\V_1\W_1) + Tr(\Y_1^T\Y_1) 
	 - Tr(\W_1^T\V_1^T\X_1^T\Y_1) - Tr(\Y_1^T\X_1\V_1\W_1) \\ 
	& + \lambda \big( Tr(\Q_1^T\V_1^T\X_1^T\X_1\V_1\Q_1) + Tr(\V_2^T\X_2^T\X_2\V_2) 
	 - Tr(\Q_1^T\V_1^T\X_1^T\X_2\V_2) - Tr(\V_2^T\X_2^T\X_1\V_1\Q_1) \big)\\
	& = \|\Y_1\|_F^2 + Tr(\W_1^T\V_1^T\X_1^T\X_1\V_1\W_1) - 2\: Tr(\W_1^T\V_1^T\X_1^T\Y_1)  \\
	& + \lambda \big( \|\X_2\V_2\|_F^2 + Tr(\Q_1^T\V_1^T\X_1^T\X_1\V_1\Q_1) - 2\:Tr(\Q_1^T\V_1^T\X_1^T\X_2\V_2) \big).
\end{align*}
$\mathcal{L}_1$ is convex with respect to $\W_1$ (and $\Q_1$), therefore the optimal parameter  $\W_1^\ast$ (and $\Q_1^\ast$) is solution to $\nabla_{\W_1}\mathcal{L}_1 = 0$ (and $\nabla_{\Q_1}\mathcal{L}_1 = 0$):
\begin{align*}
	\nabla_{\W_1}\mathcal{L}_1 = 0 & \iff 2\V_1^T\X_1^T\X_1\V_1\W_1 - 2\V_1^T\X_1^T\Y_1 = 0 \\
	& \iff \V_1^T\X_1^T\X_1\V_1\W_1 = \V_1^T\X_1^T\Y_1
\end{align*}
Recall that we supposed that the rank of $\X_1$ is equal to $D$. Because $\X_1^T\X$ and $\X_1$ have the same kernel, it follows from the "rank-nullity" theorem that $\X_1^T\X_1$ and $\X_1$ have the same rank.  
Since $\X_1^T\X_1$ is full rank, it is invertible. It follows that:
\begin{equation}
	\W_1^\ast = (\V_1^T\X_1^T\X_1\V_1)^{-1}\V_1^T\X_1^T\Y_1.
\end{equation}
In the same way,
\begin{equation}
	\Q_1^\ast = (\V_1^T\X_1^T\X_1\V_1)^{-1}\V_1^T\X_1^T\X_2\V_2.
\end{equation}
Plugging the values of $\W_1^\ast$ and $\Q_1^\ast$ in $\mathcal{L}_1$, we first notice that:
\begin{align*}
	Tr(\W_1^{\ast T}\V_1^T\X_1^T\X_1\V_1\W_1^\ast) & = Tr(\Y_1^T\X_1\V_1(\V_1^T\X_1^T\X_1\V_1)^{-1}\V_1^T\X_1^T\X_1\V_1(\V_1^T\X_1^T\X_1\V_1)^{-1}\V_1^T\X_1^T\Y_1) \\
	& = Tr(\Y_1^T\X_1\V_1(\V_1^T\X_1^T\X_1\V_1)^{-1}\V_1^T\X_1^T\Y_1) \\
	& = Tr(\Y_1^T\X_1\V_1\W_1^\ast) \\
	\mbox{and} & \\
	Tr(\Q_1^{\ast T}\V_1^T\X_1^T\X_1\V_1\Q_1^\ast) & = Tr(\V_2^T\X_2^T\X_1\V_1\Q_1^\ast).
\end{align*}
Therefore, the loss simplifies as follows:
\begin{align}
	\mathcal{L}_1 & = \|\Y_1\|_F^2 + \lambda \|\X_2\V_2\|_F^2 - Tr(\Y_1^T\X_1\V_1\W_1^\ast) - \lambda Tr(\V_2^T\X_2^T\X_1\V_1\Q_1^\ast) \nonumber \\
	& = \|\Y_1\|_F^2 - Tr(\Y_1^T\X_1\V_1(\V_1^T\X_1^T\X_1\V_1)^{-1}\V_1^T\X_1^T\Y_1) \nonumber \\
	& + \lambda \|\X_2\V_2\|_F^2 - \lambda Tr(\V_2^T\X_2^T\X_1\V_1(\V_1^T\X_1^T\X_1\V_1)^{-1}\V_1^T\X_1^T\X_2\V_2) \nonumber \\
	& = \|\Y_1\|_F^2 +  \lambda \|\X_2\V_2\|_F^2
	- Tr\big((\V_1^T\X_1^T\X_1\V_1)^{-1}\V_1^T\X_1^T(\Y_1\Y_1^T + \lambda\X_2\V_2(\X_2\V_2)^T)\X_1\V_1\big). \label{eq:L1_simplified}
\end{align}
Hence, minimizing $\mathcal{L}_1$ is equivalent at maximizing the third term of \eqref{eq:L1_simplified}, which is itself equivalent to the following optimization problem:
\begin{align}
	\max_{\V_1} & \:\:\:\: Tr\big(\V_1^T\X_1^T(\Y_1\Y_1^T + \lambda\X_2\V_2(\X_2\V_2)^T)\X_1\V_1\big) \label{eq:eq_pb} \\
	\mbox{subject to} & \:\: \V_1^T\X_1^T\X_1\V_1 = \I_K \label{eq:cst}
\end{align}
Given that $\overbrace{\X_1^T\X_1}^{=\B}$ and $\overbrace{\X_1^T(\Y_1\Y_1^T + \lambda\X_2\V_2(\X_2\V_2)^T)\X_1}^{=\A}$ are symmetric, we can actually solve \eqref{eq:eq_pb} by solving the following generalized eigenvalue problem \cite{ghojogh2019eigenvalue}:
\begin{equation}
	\A \V_1' = \B \V_1' \L \label{eq:eigen}
\end{equation}
where the columns of $\V_1' \in \mathbb{R}^{D \times D}$ are the eigenvectors of $\A$ and the diagonal elements of $\L$ its eigenvalues. Then, we can simply take the top $K$ eigenvectors of $\V_1'$ to build $\V_1$, \textit{i.e.} $\V_1 = (\V_1')_{\cdot, 1:K}$, in order to minimize $\mathcal{L}_1$. From the constraint \eqref{eq:cst}, we have that:
\begin{align*}
	\eqref{eq:cst} & \iff \V_1'^T \P_{\X_1^T\X_1} \D_{\X_1^T\X_1} \P_{\X_1^T\X_1}^T \V_1' = \I_K
\end{align*}
Therefore, for any orthogonal matrix $\P \in \mathbb{R}^{D \times D}$ such that $\V_1' = \P_{\X_1^T\X_1} \D_{\X_1^T\X_1}^{-\frac{1}{2}} \P$, the constraint \eqref{eq:cst} is satisfied. Then, we can derive from the generalized eigenvalue problem that:
\begin{align*}
	\eqref{eq:eigen} & \iff \V_1'^T \A \V_1 = \V_1'^T \B \V_1' \L \\
	& \iff \P^T \D_{\X_1^T\X_1}^{-\frac{1}{2}} \P_{\X_1^T\X_1}^T \A \P_{\X_1^T\X_1} \D_{\X_1^T\X_1}^{-\frac{1}{2}} \P = \L \\
	& \iff \P_H^T \H \P_H = \L
\end{align*}
where $\H = \D_{\X_1^T\X_1}^{-\frac{1}{2}} \P_{\X_1^T\X_1}^T \A \P_{\X_1^T\X_1} \D_{\X_1^T\X_1}^{-\frac{1}{2}}$ and the columns of $\P_H = \P$ are the eigenvectors of $\H$. Because $\H$ is symmetric, $\P_H$ is indeed orthogonal, and this yields conditions for the optimal parameters regarding modality 1. \\

\noindent
Regarding modality 2, the beginning of the proof is the same. The difference is that we can derive a simpler expression of $\H_2$ by expanding $\X_2$ from its singular value decomposition:
\begin{align*}
	\X_2 = \P_{\X_2\X_2^T} \S_{\X_2} \P_{\X_2^T\X_2}^T
\end{align*}
where $\S_{\X_2} \in \mathbb{R}^{N \times D}$ is a rectangular-diagonal matrix whose diagonal elements are the square root of $\X_2^T\X_2$ eigenvalues. Hence,
\begin{align*}
	\D_{\X_2^T\X_2}^{-\frac{1}{2}} \P_{\X_2^T\X_2}^T \X_2^T & = \D_{\X_2^T\X_2}^{-\frac{1}{2}} \underbrace{\P_{\X_2^T\X_2}^T \P_{\X_2^T\X_2}}_{= \I_{D}} \S_{\X_2}^T \P_{\X_2\X_2^T}^T \\
	& = \begin{bmatrix}
		\I_D & \boldsymbol{0}_{D \times (N-D)}
	\end{bmatrix}
	\P_{\X_2\X_2^T}^T  \\
	& = (\P_{\X_2\X_2^T})_{\cdot, 1:D}^T.
\end{align*}
Plugging the above equation into $\H_2$, we have that: 
\begin{align*}
	\H_2 & = (\P_{\X_2\X_2^T})_{\cdot, 1:D}^T \Y_2 \Y_2^T (\P_{\X_2\X_2^T})_{\cdot, 1:D} \\
	& = (\P_{\X_2\X_2^T})_{\cdot, 1:D}^T \P_{\Y_2\Y_2^T} \D_{\Y_2\Y_2^T} \P_{\Y_2\Y_2^T}^T (\P_{\X_2\X_2^T})_{\cdot, 1:D} \\
	& = (\P_{\X_2\X_2^T})_{\cdot, 1:D}^T (\P_{\Y_2\Y_2^T})_{\cdot, 1:D} \D_{\Y_2^T\Y_2} (\P_{\Y_2\Y_2^T})_{\cdot, 1:D}^T (\P_{\X_2\X_2^T})_{\cdot, 1:D}.
\end{align*}
Therefore:
\begin{align*}
	\P_{\H_2} = (\P_{\X_2\X_2^T})_{\cdot, 1:D}^T (\P_{\Y_2\Y_2^T})_{\cdot, 1:D}.
\end{align*}

\subsection{Proof of Theorem \ref{th:2}} \label{app:th2}

\noindent
The proof is made by contradiction, and is divided into five steps. The idea is to show that if $\V_1^\star$ is the optimal solution of the prediction task 1 and the alignment task, then $\sigma_{21} \geq \sigma_{1K}$, which contradicts assumption ii). \\

\noindent
\textbf{Step 1} Let us assume that $\V_1^\star$ is the optimal solution of the prediction task 1 and the alignment task, defined as in Theorem \ref{th:1}. Therefore, $\V_1^\ast$ does not depend on $\lambda$, and it follows that, with the same notations as in Theorem \ref{th:1}, $(P_{\H_1})_{\cdot, 1:K}$ does not depend on $\lambda$ either.
By definition, $(P_{\H_1})_{\cdot, 1:K}^T \H_1 (P_{\H_1})_{\cdot, 1:K}$ is equal to a diagonal matrix $\Lambda \in \mathbb{R}^{K \times K}$:
\begin{equation}
\begin{array}{ccc}
	(P_{\H_1})_{\cdot, 1:K}^T \H_1 (P_{\H_1})_{\cdot, 1:K} = \Lambda
	& \iff & \P_\Lambda^T(\Y_1\Y_1^T + \lambda \Z_2\Z_2^T) \P_\Lambda = \Lambda \\
	& \iff & \P_\Lambda^T\Y_1\Y_1^T\P_\Lambda + \lambda \P_\Lambda^T\Z_2\Z_2^T\P_\Lambda = \Lambda
\end{array}
\label{eq:lambda}
\end{equation}
where $\P_\Lambda = \X_1  \P_{\X_1^T\X_1} \D_{\X_1^T\X_1}^{-\frac{1}{2}} (P_{\H_1})_{\cdot, 1:K}$. 
The columns of $\P_\Lambda$ are the top-K eigenvectors of $(\Y_1\Y_1^T + \lambda \Z_2\Z_2^T)$. 
Note that we may find a $\lambda$ for which $\Lambda$ is diagonal while $\P_\Lambda^T\Y_1\Y_1^T\P_\Lambda$ and $\P_\Lambda^T\Z_2\Z_2^T\P_\Lambda$ are not. 
However, eq. \eqref{eq:lambda} holds for any $\lambda$, and $\P_\Lambda$ does not depend on $\lambda$, therefore the columns of $\P$ are also the top-K eigenvectors of $\Y_1\Y_1^T$ and $\Z_2\Z_2^T$.
It results that the intersection of the top-K eigenspaces of $\Z_2\Z_2^T$ and $\Y_1\Y_1^T$ is of dimension $K$. \\

\noindent
\textbf{Step 2} Let us show that assumption i) implies that the intersection of the top-K eigenspaces of $\Y_2\Y_2^T$ and $\X_2\X_2^T$ is of dimension $K$. 
Let us denote $\hat{\Y}_2 = \X_2(\X_2^T\X_2)^{-1}\X_2^T\Y_2$ the OLS prediction of $\Y_2$ from $\X_2$. 
First, we notice that we always have the following:
\begin{align}
	\|\Y - \hat{\Y}\|_F^2 & = Tr[(\Y - \hat{\Y})^T(\Y - \hat{\Y})]
	= \|\Y\|_F^2 - Tr[\Y^T\hat{\Y}] \\ 	\label{eq:normtrace}
	& \mbox{ because } \hat{\Y}^T\hat{\Y} = \Y^T\hat{\Y} \nonumber.
\end{align}
Second, we can show from the eigenvalue value decomposition of $\X_2^T\X_2$ and from the singular value decomposition of $\X_2$ that:
\begin{align*}
	\Y^T\hat{\Y} & = \Y^T \X (\P_{\X^T\X} \D_{\X^T\X} \P_{\X^T\X}^T)^{-1} \X^T \Y \\
	& = \Y_2^T \underbrace{\X \P_{\X^T\X}}_{=(\P_{\X\X^T})_{\cdot, 1:D} \D_{\X^T\X}^{\frac{1}{2}}} \D_{\X^T\X}^{-1} \underbrace{\P_{\X^T\X}^T \X^T}_{=\D_{\X^T\X}^{\frac{1}{2}}(\P_{\X\X^T})_{\cdot, 1:D}^T} \Y \\
	& = \Y^T (\P_{\X\X^T})_{\cdot, 1:D} (\P_{\X\X^T})_{\cdot, 1:D}^T \Y.
\end{align*}
Defining $\hat{\tilde{\Y}}_{2K} = \tilde{\Y}_{2K, 0}^0({\tilde{\Y}_{2K, 0}}^T\tilde{\Y}_{2K,0})^{-1}{\tilde{\Y}_{2K,0}}^T\Y$, where $\tilde{\Y}_{2K,0}$ is composed of $\tilde{\Y}_{2K}$ for the first $K$ rows, and completed by rows of zeros. 
It follows that assumption i) is equivalent to:
\begin{align*}
    \sigma_{22} \geq \sigma_{2K} & \iff NC_2(1-\sigma_{22}) \leq NC_2(1 - \sigma_{2K}) \\
    & \iff \|\Y_2 - \hat{\Y}_2\|_F^2 \leq \|\Y_2 - \hat{\tilde{\Y}}_{2K}\|_F^2 \\
    & \iff \|\Y_2\|_F^2 - Tr[\Y_2^T \hat{\Y}_2] \leq  \|\Y_2\|_F^2 - \underbrace{Tr[\hat{\tilde{\Y}}_{2K}^T \hat{\tilde{\Y}}_{2K}]}_{=Tr[(\D_{\Y_2\Y_2^T})_{1:K}]} \\
    & \iff Tr[\Y_2^T\hat{\Y}_2] \geq Tr[(\D_{\Y_2\Y_2^T})_{1:K}] \\
    & \iff Tr[(\P_{\X_2\X_2^T})_{\cdot, 1:D}^T \Y_2  \Y_2^T (\P_{\X_2\X_2^T})_{\cdot, 1:D} ] \geq Tr[(\D_{\Y_2\Y_2^T})_{1:K}] \\
    & \iff Tr[(\P_{\X_2\X_2^T})_{\cdot, 1:D}^T (\P_{\Y_2\Y_2^T})_{\cdot, 1:D} (\D_{\Y_2\Y_2^T})_{1:D} (\P_{\Y_2\Y_2^T})_{\cdot, 1:D}^T (\P_{\X_2\X_2^T})_{\cdot, 1:D} ] \\
    & \:\:\:\:\:\:\:\:\:\:\:\:\:\:\:\:\:\:\:\:\:\:\:\:\:\:\:\:\:\:\:\:\:\:\:\: \geq Tr[(\D_{\Y_2\Y_2^T})_{1:K}] \\
    &  \iff Tr[P_{\otimes} (\D_{\Y_2\Y_2^T})_{1:D}P_{\otimes}^T] \geq Tr[(\D_{\Y_2\Y_2^T})_{1:K}].
\end{align*}
where $P_{\otimes} = (\P_{\X_2\X_2^T})_{\cdot, 1:D}^T (\P_{\Y_2\Y_2^T})_{\cdot, 1:D}$ is an orthogonal matrix.
Since, from the original assumptions in Section \ref{sec: setting}, the eigenvalues of $\Y_2\Y_2^T$ are such that $\lambda_k > \sum_{i=k+1} \lambda_i$ for every $k$, this proves that the intersection of the top-K eigenspaces of $\X_2\X_2^T$ and $\Y_2\Y_2^T$ is of dimension K. \\

\noindent
\textbf{Step 3} Let us show that the intersection of the top-K eigenspaces of $\Z_2\Z_2^T$ and $\X_2\X_2^T$ is of dimension $K$.
From Theorem \ref{th:1}, we have that:
\begin{align*}
	\V_2^\ast = \P_{\X_2^T\X_2} \D_{\X_2^T\X_2}^{-\frac{1}{2}} (\P_{\H_2})_{\cdot, 1:K}\R
\end{align*}
where $\R \in \mathbb{R}^{K \times K}$ is a rotation matrix. Recalling that $\P_{\H_2} = (\P_{\X_2\X_2^T})_{\cdot, 1:D}^T (\P_{\Y_2\Y_2^T})_{\cdot, 1:D}$, and using the singular value decomposition of $\X_2 = \P_{\X_2\X_2^T} \D_{\X_2^T\X_2}^{\frac{1}{2}}  \P_{\X_2^T\X_2}^T$, we can derive that:
\begin{align*}
    \Z_2 = \X_2 \V_2^\ast & = \X_2 \P_{\X_2^T\X_2} \D_{\X_2^T\X_2}^{-\frac{1}{2}} (\P_{\X_2\X_2^T})_{\cdot, 1:D}^T (\P_{\Y_2\Y_2^T})_{\cdot, 1:K} \R \\
	& = (\P_{\X_2\X_2^T})_{\cdot, 1:D} (\P_{\X_2\X_2^T})_{\cdot, 1:D}^T (\P_{\Y_2\Y_2^T})_{\cdot, 1:K} \R
\end{align*}
which leads to:
\begin{align*}
	\Z_2 \Z_2^T = (\P_{\X_2\X_2^T})_{\cdot, 1:D} & (\P_{\X_2\X_2^T})_{\cdot, 1:D}^T (\P_{\Y_2\Y_2^T})_{\cdot, 1:K} \underbrace{\R \R^T}_{= \I_K} \\
	& \cdot (\P_{\Y_2\Y_2^T})_{\cdot, 1:K}^T (\P_{\X_2\X_2^T})_{\cdot, 1:D} (\P_{\X_2\X_2^T})_{\cdot, 1:D}^T.
\end{align*}
From step 2, we have that the intersection of the top-K eigenspaces of $\X_2\X_2^T$ and $\Y_2\Y_2^T$ is of dimension K. Under the assumption that the top-$K$ eigenvalues of $\X_2\X_2^T$ have a multiplicity of one, this means that there exists a permutation matrix $\Q_2 \in \mathbb{R}^{N \times N}$ such that:
\begin{align*}
	\P_{\Y_2\Y_2^T} = \P_{\X_2\X_2^T} \Q_2.
\end{align*}
Therefore,
\begin{align*}
	(\P_{\X_2\X_2^T})_{\cdot, 1:D}^T (\P_{\Y_2\Y_2^T})_{\cdot, 1:K} & = (\P_{\X_2\X_2^T})_{\cdot, 1:D}^T \P_{\X_2\X_2^T} (\Q_2)_{\cdot, 1:K}
	= \begin{bmatrix}
		(\Q_2)_{\cdot, 1:K} \\
		\boldsymbol{0}_{(D-K) \times K}
	\end{bmatrix}
\end{align*}
and
\begin{align*}
	\Z_2 \Z_2^T & = (\P_{\X_2\X_2^T})_{\cdot, 1:D} \begin{bmatrix}
		\I_K & \boldsymbol{0}_{K \times (D-K)} \\
		\boldsymbol{0}_{(D-K) \times K} & \boldsymbol{0}_{(D-K) \times (D-K)}	
\end{bmatrix}	 (\P_{\X_2\X_2^T})_{\cdot, 1:D}^T \\
	& = (\P_{\X_2\X_2^T})_{\cdot, 1:K} \I_K (\P_{\X_2\X_2^T})_{\cdot, 1:K}^T.
\end{align*}
Thus, the intersection of the top-K eigenspaces of $\X_2\X_2^T$ and $\Z_2\Z_2^T$ is of dimension $K$. \\

\noindent
\textbf{Step 4}  From step 1, we have that the intersection of the top-K eigenspaces of $\Y_1\Y_1^T$ and $\Z_2 \Z_2^T$ is of dimension K, and in particular that  $\Y_1\Y_1^T$ and $\Z_2 \Z_2^T$ share the same top-K eigenvectors. Therefore, the intersection of the top-K eigenspaces of $\Z_2\Z_2^T$ and $\X_2 \X_2^T$ is the same space as the intersection of the top-K eigenspaces of $\Y_1\Y_1^T$ and $\X_2 \X_2^T$.
From step 2, we have that the intersection of the top-K eigenspaces of $\Z_2\Z_2^T$ and $\X_2 \X_2^T$ is of dimension K, hence the desired result; the intersection of the top-K eigenspaces of $\Y_1\Y_1^T$ and $\X_2 \X_2^T$ is of dimension K.  \\

\noindent
\textbf{Step 5} Let us conclude the proof by showing that if the intersection of the top-K eigenspaces of $\Y_1\Y_1^T$ and $\X_2 \X_2^T$ is of dimension K, then $\sigma_{21} \geq \sigma_{1K}$.
Let us denote by $\hat{\Y}_{21}$ the OLS predictions of $\Y_1$ computed from $\X_2$.
Recalling (10) and the singular value decomposition of $\Y_1$, we have that:
\begin{align*}
    \Y_1^T\hat{\Y}_{21} = & \P_{\Y_1^T\Y_1} \D_{\Y_1^T\Y_1}^{\frac{1}{2}} (\P_{\Y_1\Y_1^T})_{\cdot, 1:C_1}^T (\P_{\X_2\X_2^T})_{\cdot, 1:D_2} \\
    & \cdot (\P_{\X_2\X_2^T})_{\cdot, 1:D_2}^T (\P_{\Y_1\Y_1^T})_{\cdot, 1:C_1} \D_{\Y_1^T\Y_1}^{\frac{1}{2}} \P_{\Y_1^T\Y_1}^T.
\end{align*}
From step 4, we have that the intersection of the top-K eigenspaces of $\Y_1\Y_1^T$ and $\X_2 \X_2^T$ is of dimension K. Recalling the assumptions from Section \ref{sec: setting}, the top-$K$ eigenvalues of $\X_2\X_2^T$ have a multiplicity of one, there exists a permutation matrix $\Q_1 \in \mathbb{R}^{N \times N}$ such that:
\begin{align*}
	\P_{\Y_1\Y_1^T} = \P_{\X_2\X_2^T} \Q_1.
\end{align*}
Therefore,
\begin{align*}
    (\P_{\Y_1\Y_1^T})_{\cdot, 1:C_1}^T (\P_{\X_2\X_2^T})_{\cdot, 1:D_2} & = \begin{bmatrix} (\P_{\Y_1\Y_1^T})_{\cdot, 1:K}^T \\ (\P_{\Y_1\Y_1^T})_{\cdot, K+1:C_1}^T  \end{bmatrix} \begin{bmatrix} (\P_{\X_2\X_2^T})_{\cdot, 1:K} & (\P_{\X_2\X_2^T})_{\cdot, K+1:D_2} \end{bmatrix} \\
    & = \begin{bmatrix} (\P_{\Y_1\Y_1^T})_{\cdot, 1:K}^T (\P_{\X_2\X_2^T})_{\cdot, 1:K} & (\P_{\Y_1\Y_1^T})_{\cdot, 1:K}^T (\P_{\X_2\X_2^T})_{\cdot, K+1:D_2} \\
    (\P_{\Y_1\Y_1^T})_{\cdot, K+1:C_1}^T (\P_{\X_2\X_2^T})_{\cdot, 1:K} & (\P_{\Y_1\Y_1^T})_{\cdot, K+1:C_1}^T (\P_{\X_2\X_2^T})_{\cdot, K+1:D_2}
    \end{bmatrix} \\
    & = \begin{bmatrix} (\Q_1)_{K \times K} & \boldsymbol{0}_{K \times (D_2-K)} \\
    \boldsymbol{0}_{(C_1-K) \times K} & (\P_{\Y_1\Y_1^T})_{\cdot, K+1:C_1}^T (\P_{\X_2\X_2^T})_{\cdot, K+1:D_2} \end{bmatrix}.
\end{align*}
It results that
\begin{align*}
    \Y_1^T\hat{\Y}_{21} & = \P_{\Y_1^T\Y_1} \D_{\Y_1^T\Y_1}^{\frac{1}{2}} \begin{bmatrix} \boldsymbol{I} & \boldsymbol{0} \\
    \boldsymbol{0} & \A \end{bmatrix} \D_{\Y_1^T\Y_1}^{\frac{1}{2}} \P_{\Y_1^T\Y_1}^T \\
    & = (\P_{\Y_1^T\Y_1})_{\cdot, 1:K} (\D_{\Y_1^T\Y_1})_{1:K}^{\frac{1}{2}} (\D_{\Y_1^T\Y_1})_{1:K}^{\frac{1}{2}} (\P_{\Y_1^T\Y_1})_{\cdot, 1:K}^T \\
    & + (\P_{\Y_1^T\Y_1})_{\cdot, K+1:C_1} (\D_{\Y_1^T\Y_1})_{K+1:C_1}^{\frac{1}{2}} \A (\D_{\Y_1^T\Y_1})_{K+1:C_1}^{\frac{1}{2}} (\P_{\Y_1^T\Y_1})_{\cdot, K+1:C_1}^T
\end{align*}
where $\A = (\P_{\Y_1\Y_1^T})_{\cdot, K+1:C_1}^T (\P_{\X_2\X_2^T})_{\cdot, K+1:D_2} (\P_{\X_2\X_2^T})_{\cdot, K+1:D_2}^T (\P_{\Y_1\Y_1^T})_{\cdot, K+1:C_1}$.
Hence,
\begin{align*}
    Tr[\Y_1^T\hat{\Y}_{21}] = \|\tilde{\Y}_{1K} \|_F^2 + \| (\P_{\X_2\X_2^T})_{\cdot, K+1:D_2}^T (\P_{\Y_1\Y_1^T})_{\cdot, K+1:C_1} (\D_{\Y_1^T\Y_1})_{K+1:C_1}^{\frac{1}{2}} (\P_{\Y_1^T\Y_1})_{\cdot, K+1:C_1}^T \|_F^2
\end{align*}
where $\tilde{\Y}_{1K} = \Y_1 (\P_{\Y_1^T\Y_1})_{\cdot, 1:K}$ is the projection of $\Y_1$ on the top-K eigenspaces of $\Y_1^T\Y_1$.
Thus, we deduce that:
\begin{align*}
    \|\Y_1 - \hat{\Y}_{21}\|_F^2 & =  \|\Y_1\|_F^2 - Tr[\Y_1^T\hat{\Y}_{21}] \\
    & \leq \|\Y_1\|_F^2 - \|\tilde{\Y}_{1K} \|_F^2 = \epsilon_{1K}
\end{align*}
where $\epsilon_{1K}$ is the training error of the OLS estimator for predicting $\Y_1$ from $\tilde{\Y}_1$.
Therefore, the training error of the OLS estimator for predicting $\Y_1$ from $\X_2$ is lower or equal than $\epsilon_{1K}$.
By definition, this implies that $\sigma_{21} \geq \sigma_{1K}$.
This contradicts the assumption that $\sigma_{21} < \sigma_{1K}$.
By contradiction, we conclude that $\V_1^\star$ does not minimize the prediction task 1 and the alignment task.
Furthermore, since $\V_1^\star$ does not minimize the prediction task 1, it necessarily implies that $\Z_1$ is $\sigma_{11}^z$-informative of $\Y_1$, with $\sigma_{11}^z < \sigma_{11}$.


\begin{remark}
In the proof, we only need in step 4 to prove that the intersection of the top-K eigenspaces of $\Y_1\Y_1^T$ and $\X_2 \X_2^T$ is of dimension K. However, a much stringer result holds: using steps 2 and 3, one can actually show that the top-K eigenspaces of $\Y_1\Y_1^T$, $\Z_2\Z_2^T$ and $\X_2 \X_2^T$ actually generate the exact same subspace.
\end{remark}

\end{document}